\documentclass{article}

\usepackage[margin=1in]{geometry}
\usepackage{natbib}                 % \citep (ICLR style provided this; article needs it explicitly)
\usepackage{amsmath,amssymb,amsthm}
\usepackage{graphicx}
\usepackage{booktabs}
\usepackage{xcolor}
\usepackage[hyphens]{url}
\usepackage{hyperref}
\hypersetup{colorlinks=true,linkcolor=blue,citecolor=blue,urlcolor=blue}
\graphicspath{{figures/}}

\newcommand{\epscert}{\epsilon_{\mathrm{cert}}}
\newcommand{\Ccal}{\mathcal{C}}
\newcommand{\Ical}{\mathcal{I}}
\newcommand{\ICVop}{\mathrm{ICV}}
\newcommand{\bUCB}{b_h^{\mathrm{UCB}}}
\newcommand{\Teq}{T_{\mathrm{eq}}}
\newcommand{\Tdrift}{T_{\mathrm{drift}}}
\newcommand{\Tlam}{T_{\lambda}}
\newcommand{\Uclu}{U^{\mathrm{cluster}}_{0.95}}
\newcommand{\Lsafe}{L_{\mathrm{safe}}}
\providecommand{\norm}[1]{\left\lVert #1 \right\rVert}

\newtheorem{proposition}{Proposition}
\newtheorem{corollary}{Corollary}
\theoremstyle{remark}
\newtheorem{remark}{Remark}
\theoremstyle{plain}

\title{Certified World Models as Sensing Clocks:\\
Drift-Aware Deadlines for Active Perception}
\author{Hongbo Wang}
\date{}

\begin{document}
\maketitle

\begin{abstract}
Certified world models estimate how long their predictions remain valid. We turn this validity horizon into an operational \textbf{sensing clock}: a rule for when an agent should stop coasting and re-sense. Starting from an audited equivariant world model, we derive a deadline for no-sensing intervals and show that deployable deadlines in learned world models must be \textbf{drift-aware}: on-manifold Lyapunov rates alone overestimate coasting validity, while calibrated native rollout-drift envelopes carry the deployed guarantee. On a frozen 3D VN-JEPA model, the resulting clock controls held-out interval-simultaneous certificate violation across seeds and data shards. In a cue-conditioned theorem-bed (a synthetic bench where all schedulers share the exact model, isolating the scheduling rule), the clock remains valid on the deployment distribution and substantially reduces eventful-tail violations relative to exact-mixture expected-belief scheduling at matched sensing budget. We also report limits: in the short-horizon frozen VN-JEPA regime, empirical conformal horizons match the deployed clock on validity and budget, and a partial-reset exploration finds no clean budget-matched advantage for the spectral term. Thus the contribution is a \textbf{certified sensing-clock primitive and drift-aware deployment method}, not a claim that spectral clocks empirically dominate all non-spectral schedulers.
\end{abstract}

\section{Introduction}
\label{sec:intro}

A world model that predicts the future still does not know \emph{when its own predictions expire}. An agent coasting on open-loop predictions must eventually re-sense; sensing too late forfeits the model's validity guarantee, sensing too early wastes a scarce budget. Fixed re-planning periods ignore the model's actual reliability, and purely reactive monitors act only \emph{after} a residual has already grown. We study the missing primitive:

\begin{quote}
\textbf{A certified world model should expose its validity horizon as an operational \emph{sensing deadline} --- a clock that tells the agent when to re-sense.} When the certified horizon counts down to zero, the agent stops coasting and senses.
\end{quote}

This completes a broader program over certified world models: \emph{when to trust} a prediction (prior certified-horizon work~\citep{wang2026certifiedwm}), \emph{when to replan}, and --- here --- \textbf{when to sense}. Figure~\ref{fig:clock} sketches the primitive.

\paragraph{Crowded-field positioning.}
Active perception, value-of-information, $\rho$-POMDP planning with guarantees, active-inference / expected-free-energy scheduling, and conformal sensing all schedule observations. Our distinction is narrow and explicit: we contribute the \textbf{first proactive, analytically-derived certified-horizon sensing trigger} computed from an audited (equivariant) world model --- with the model's spectral expansion rate as the \emph{audit interface} the horizon is derived through, and a calibrated drift envelope carrying the deployed guarantee --- contrasted against reactive information-gain greedy schedulers, fixed-horizon polling, and conformal residual monitors that lack a predictive-validity certificate. This is a \emph{crowded-field differentiation paper with a new primitive}, not an empty-niche claim.

\paragraph{Contributions.}
We make four positive contributions and state one boundary explicitly.
\begin{enumerate}
\item \textbf{The sensing-clock primitive} (\texttt{certificate = sensing clock}): a certified coasting deadline that controls interval-simultaneous certificate violation (\S3, Proposition~\ref{prop:drift}, \S4).
\item \textbf{A drift-aware deployment method} --- deployable deadlines are not $\lambda$-only: on-manifold Lyapunov rates over-state the deployable horizon by roughly an order of magnitude, and the load-bearing term is a calibrated native rollout-drift envelope (\S3, \S5.4).
\item \textbf{Instantiation on a real equivariant world model} --- on a frozen 3D VN-JEPA model the clock controls held-out interval-ICV, confirmed across seeds and data shards (\S4, \S5.1).
\item \textbf{A reactive-contrast theorem-bed result} --- under a shared model, expected-belief reactive scheduling (MB-EIG) under-prices eventful-tail risk: the certified clock stays valid on the deployment distribution and cuts eventful-tail violations from $0.36$ to $0.16$ at matched budget, with MB-EIG needing $\approx3\times$ budget to recover the tail (\S4, \S5.2).
\end{enumerate}

\paragraph{Boundary (stated, not hidden).}
We do \emph{not} claim empirical dominance over non-spectral schedulers. An exact-synthetic lead-time test returns null (\S5.3); in the learned regime a non-spectral empirical conformal horizon matches the deployed clock on validity and budget (\S5.4); and a partial-reset exploration finds no budget-matched edge for the spectral term (\S5.5). These results narrow --- not refute --- the primitive (\S6).

\begin{figure}[t]
\centering
\includegraphics[width=\linewidth]{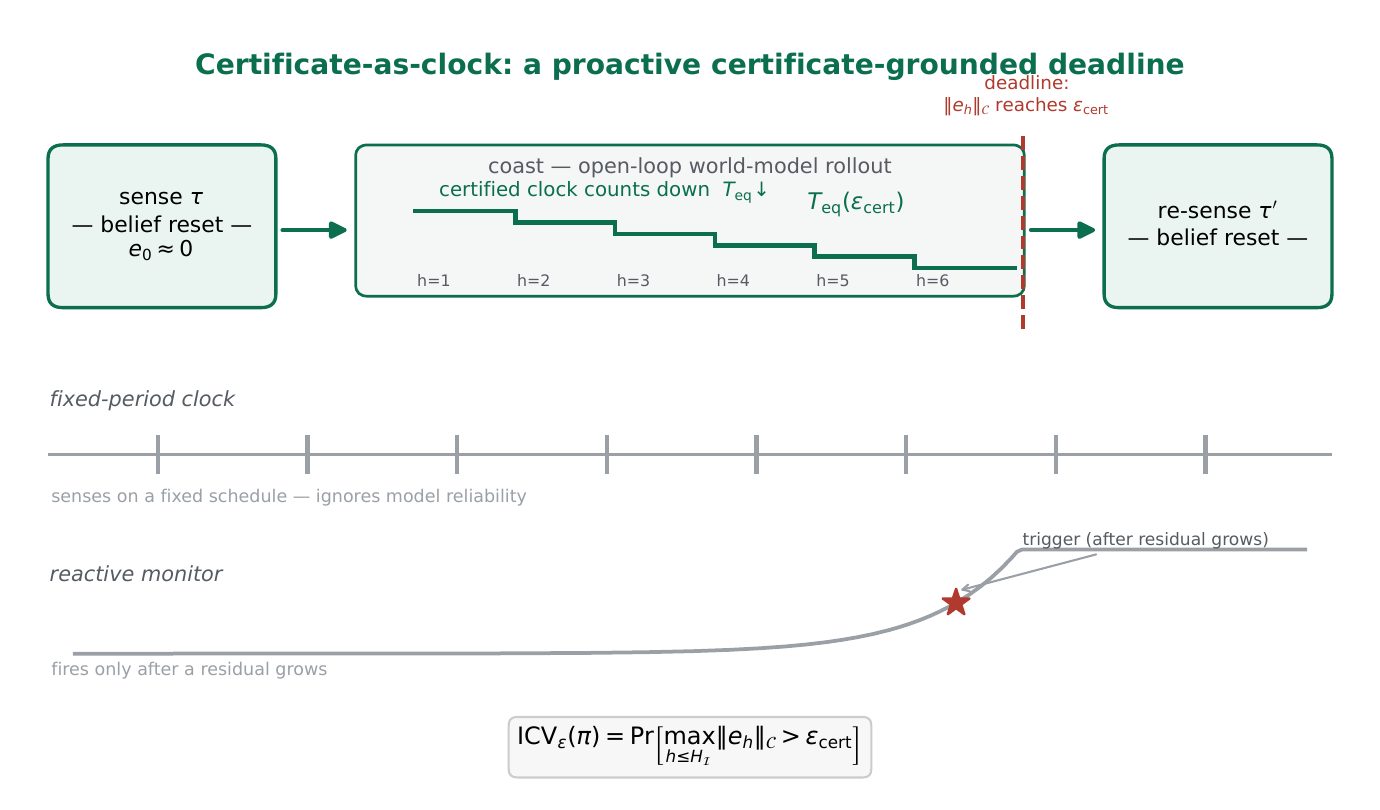}
\caption{\textbf{The certificate-as-clock primitive.} A certified world model exposes its validity horizon as an operational sensing deadline: after a sense resets belief ($e_0\approx0$), the model coasts open-loop while the certificate holds and re-senses when the certified clock $\Teq(\epscert)$ expires (the coasting error norm $\norm{e_h}_{\Ccal}$ reaches the certified tolerance $\epscert$). Unlike a fixed-period clock (insensitive to model reliability) or a reactive monitor (fires only after a residual grows), the deadline is proactive and certificate-grounded.}
\label{fig:clock}
\end{figure}

\section{Problem Setup}
\label{sec:setup}

\paragraph{Coasting intervals.}
An agent alternates \emph{sensing} (resetting belief from an observation) and \emph{coasting} (predicting open-loop with the world model). A coasting interval $\Ical=(\tau,\tau')$ spans the steps between two senses; $H_{\Ical}=\tau'-\tau-1$ is its length. The deployed policy is the deadline rule that sets $\tau'$ given the state at $\tau$.

\paragraph{Interval-simultaneous certificate violation (primary metric).}
Let $\norm{\cdot}_{\Ccal}$ be a certificate norm on the world model's latent (for VN-JEPA, an irrep/block norm; \S3). The pre-registered primary metric is the probability that the coasting prediction error exceeds the certified tolerance $\epscert$ \emph{anywhere} in the interval:
\begin{equation}
\ICVop_\epsilon(\pi)=\Pr_{\Ical}\Big[\max_{1\le h\le H_{\Ical}}\norm{\hat z_{\tau+h\mid\tau}-z_{\tau+h}}_{\Ccal}>\epscert\Big],
\qquad \bar U:=\Uclu\big(\widehat{\ICVop}_\epsilon\big),
\label{eq:icv}
\end{equation}
with a cluster-bootstrap upper bound over seeds/episodes (intervals within an episode are not independent); we write $U_{95}$ for $\bar U$ in tables and results. The pre-registered target is $\bar U\le\alpha_{\mathrm{target}}+\delta=0.10+0.05=0.15$.

\paragraph{Sensing budget.}
$B(\pi)=\mathbb{E}[\#\text{sense}/T]$, the average sensing rate. Comparisons are made at \emph{matched budget} (thresholds calibrated on a calibration split, frozen, evaluated once on test).

\paragraph{Lead-time (optional; native traces permitting).}
For safety-critical deployments one also wants \emph{actionable} lead-time $\Lsafe=t_{\mathrm{last\text{-}safe}}-\tau_{\mathrm{trigger}}$. This requires native hazard / last-safe / intervention semantics, which the VN-JEPA coasting corpus does not contain (\S5.4); we test lead-time only on a synthetic bench (\S5.3) and explicitly defer the learned-regime lead-time question.

\section{Method}
\label{sec:method}

\paragraph{Spectral horizon (naive).}
If coasting error grew purely by a local expansion rate $\lambda$ from a post-sense error $e_0$, the deadline would be $\Tlam(\epsilon)\sim\log(\epsilon/e_0)/\lambda$. On synthetic linear-Gaussian error dynamics, this oracle spectral horizon provides a \emph{calibrated upper bound} on interval-simultaneous certificate violation, underlying the mechanism unit tests (\S4.1).

\paragraph{Drift-aware correction (load-bearing).}
On a real frozen world model the on-manifold rate alone over-states the deployable horizon, because the coasting error carries native rollout drift, model bias, and off-manifold departure that $\lambda$ does not see. We bound the realized error by
\begin{equation}
r_h\le \bUCB+C\,e^{\lambda h}e_0+\eta_h,
\qquad
\Teq=\sup\bigl\{h:\ \bUCB+C\,e^{\lambda h}e_0+\eta_h\le\epscert\bigr\},
\label{eq:drift-bound}
\end{equation}
where $\bUCB$ is a calibrated \textbf{native rollout-drift envelope} (a held-out high quantile of the coasting error at horizon $h$), $C e^{\lambda h}e_0$ the audited local-expansion term, and $\eta_h$ a finite-sample slack. \emph{The audited Lyapunov rate ranks local on-manifold expansion; deployable coasting validity is governed by the sum of local expansion and native rollout drift.}

\begin{proposition}[drift-aware interval certificate]
\label{prop:drift}
Suppose on the calibration distribution the bound $\bUCB+C\,e^{\lambda h}e_0+\eta_h$ dominates the realized coasting error $r_h$ simultaneously over $1\le h\le H$ with probability at least $1-\alpha$. Then the deadline $\Teq=\sup\{h:\ \bUCB+C\,e^{\lambda h}e_0+\eta_h\le\epscert\}$ controls interval-simultaneous certificate violation at level $\alpha$, up to the calibration failure probability: $\ICVop_{\epscert}(\pi_{\Teq})\le\alpha$.
\end{proposition}

\begin{proof}[Proof sketch]
On the event that the bound holds simultaneously, for all $h\le\Teq$ we have $r_h\le\bUCB+Ce^{\lambda h}e_0+\eta_h\le\epscert$, so no interval-simultaneous violation occurs; the violation event is contained in the calibration-failure event of probability $\le\alpha$.
\end{proof}

\noindent This restates the deployed calibration protocol as a certificate; it makes no claim that $\lambda$ improves the bound --- in the full-resense regime $e_0\approx0$ and the bound is drift-dominated.

\paragraph{Deployed form in the full-resense regime.}
With full re-sensing the post-sense error $e_0\approx0$, so $Ce^{\lambda h}e_0$ vanishes and the deployed deadline reduces to the \textbf{drift-envelope clock}
\begin{equation}
\Tdrift(\epsilon)=\sup\{h:\ \bUCB\le\epscert\}.
\label{eq:tdrift}
\end{equation}
We confirm this is exactly the deployed deadline on the frozen VN-JEPA (\S5.4). Consequently the on-manifold spectrum $\lambda$ acts as a \textbf{local-expansion audit / theoretical interface}, not an active term in the deployed horizon. (Empirically, audited on-manifold $\lambda\approx0.07$; a naive $\lambda$-only horizon over-states the deployable horizon by roughly an order of magnitude, while the drift envelope yields the deployed ${\sim}2$--$3$ step horizon.)

\begin{figure}[t]
\centering
\includegraphics[width=\linewidth]{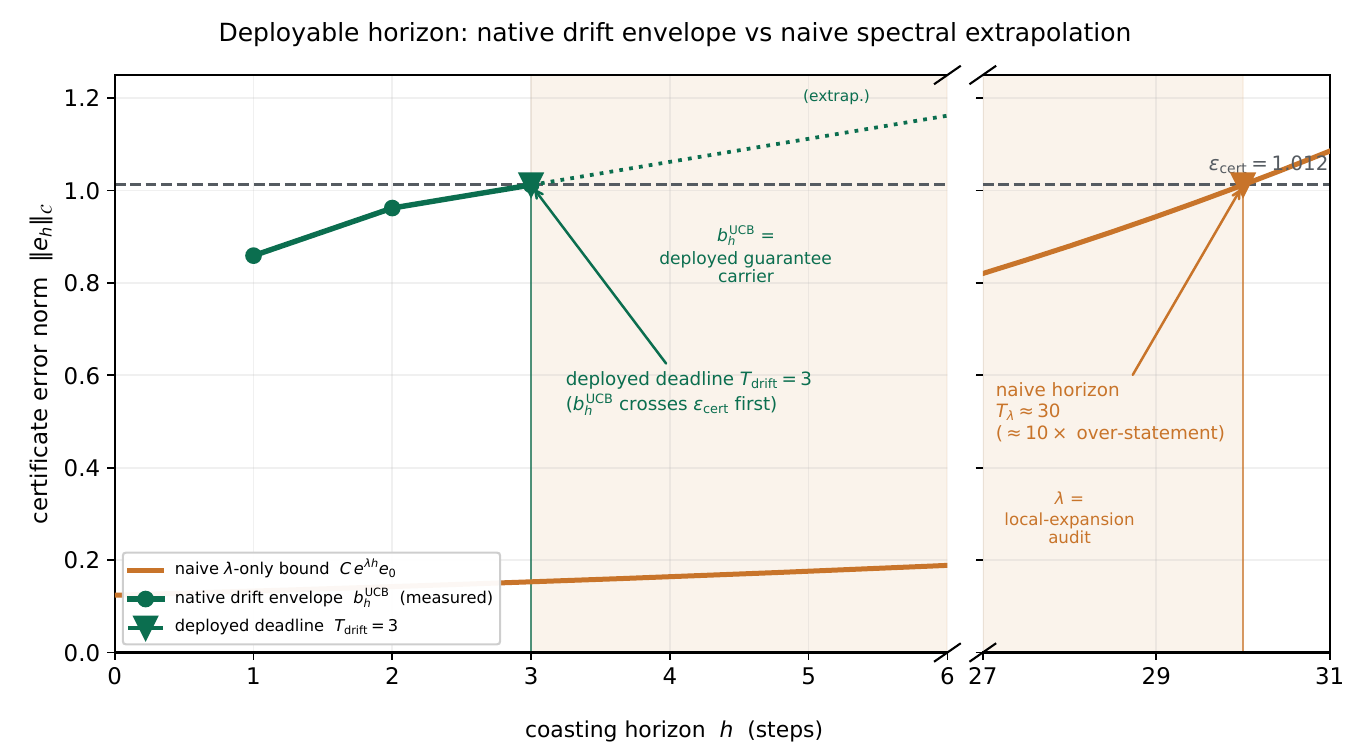}
\caption{\textbf{Why a spectral-only horizon over-states the deadline.} Coasting error vs horizon at fixed $\epscert$ (broken $h$-axis: left = measured regime, right = the naive crossing). On-manifold $\lambda$ ranks local expansion, but native rollout drift determines deployable coasting validity in the frozen VN-JEPA regime: the calibrated drift envelope $\bUCB$ crosses $\epscert$ first, at the deployed $\Tdrift\approx2$--$3$ steps, whereas the naive $\lambda$-only bound $C e^{\lambda h}e_0$ (audited $\lambda\approx0.07$) crosses roughly an order of magnitude later. The deployed deadline is set by whichever term reaches $\epscert$ first; in the full-resense regime ($e_0\approx0$) that is always the drift envelope. The plotted $\epscert=1.012$ is a representative audit value for the mechanism schematic; deployed tolerances are calibrated independently per split, as reported in Table~\ref{tab:s1b}.}
\label{fig:drift}
\end{figure}

\paragraph{Calibration protocol.}
On VN-JEPA the certificate norm is the per-block normalized maximum over irrep blocks $b$, $\norm{e}_{\Ccal}=\max_b \norm{e_b}_2/s_b$, with per-block scales $s_b$. These scales, the tolerance $\epscert$, the drift envelope $\bUCB$, and any conformal slack are fixed on a calibration split, where $\epscert$ is chosen as the smallest tolerance whose cal interval-ICV upper bound meets $\alpha_{\mathrm{target}}$ at the induced deadline. The test split is evaluated once with no retuning. The spectral audit is performed \emph{on the native on-manifold trajectory} (free-iteration of a short-horizon residual predictor produces an off-manifold expansion artifact, $\lambda\approx3.5$ vs the on-manifold $\approx0.07$, and is used only as an off-manifold stress diagnostic).

\section{Experiments}
\label{sec:exp}

A ladder isolates one variable per rung; failure attribution stays clean. Table~\ref{tab:ladder} maps the rungs; results follow in \S5.

% Table 1 -- experimental ladder (reviewer's map)
\begin{table}[t]
\centering
\small
\caption{Experimental ladder --- the reviewer's map. The ladder isolates one variable per rung; confirmatory rungs (0A\,$\to$\,1b-lite\,$\to$\,2A) build the positive claim, while honest negatives (2B, 2C-V, E0) define the boundary.}
\label{tab:ladder}
\begin{tabular}{llll}
\toprule
Rung & Question & Result & Status \\
\midrule
Stage 0A      & oracle deadline mechanism (synthetic)  & PASS           & confirmatory \\
Stage 1a      & finite-data audited spectrum           & PASS           & confirmatory \\
Stage 1b-lite & frozen 3D VN-JEPA clock                 & PASS-CANDIDATE & seed + shard \\
Stage 2A      & expected-belief tail risk               & PASS           & mechanism \\
Stage 2B      & exact-synthetic lead-time               & NULL           & honest negative \\
Stage 2C-V    & learned validity/budget replaceability  & warning        & conformal matches \\
E0            & spectral-term deployment edge           & NULL           & exploratory \\
\bottomrule
\end{tabular}
\end{table}

\noindent Per-rung detail:
\begin{itemize}
\item \textbf{Stage 0A} --- oracle deadline mechanism on a synthetic error-dynamics bench: the union-bound worst-case deadline controls interval-ICV ($U_{95}=0.040$ and $0.073$ on two spec variants, both $\le0.15$), and the bench is falsifiable (ICV grows monotonically with the sensing period, with the certified deadline at the knee). \textbf{PASS.}
\item \textbf{Stage 1a} --- replace the oracle spectrum with a finite-data audited estimate (OLS + bootstrap UCB): validity survives estimation error ($U_{95}\le0.073$) at near-oracle budget (ratio $\le1.11$), down to $n\approx5$ audit samples. \textbf{PASS.}
\item \textbf{Stage 1b-lite} --- a frozen 3D VN-JEPA equivariant world model (an SO(3)-equivariant vector-neuron JEPA from the same certified-horizon program as \citealt{wang2026certifiedwm}): audited spectrum + native drift envelope $\to$ drift-aware deadline $\to$ held-out native coasting interval-ICV. \textbf{PASS-CANDIDATE}, confirmed across seed and shard.
\item \textbf{Stage 2A} --- a cue-conditioned reactive-contrast theorem-bed where all schedulers share the same model; tests whether expected-belief reactive under-prices the eventful tail. \textbf{mechanism PASS.}
\item \textbf{Stage 2B} --- a state-dependent lead-time bench. \textbf{NULL.}
\item \textbf{Stage 2C} --- learned-regime replaceability: a feasibility audit (2C-0: native traces lack hazard semantics, so the full three-axis replaceability test is blocked) and a validity/budget test (2C-V). \textbf{validity-only warning.}
\item \textbf{E0 (exploratory)} --- a partial-reset regime to activate the spectral term. \textbf{NULL.}
\end{itemize}

\section{Results}
\label{sec:results}

\subsection{Stage 1b-lite --- frozen 3D VN-JEPA (seed + shard confirmed)}
\label{sec:r-1b}

% Table 2 -- Stage 1b-lite seed + shard confirmation
\begin{table}[t]
\centering
\small
\caption{Stage 1b-lite --- seed + shard confirmation on a frozen 3D VN-JEPA model. Each row calibrates $\epscert$ on the calibration split and evaluates once on held-out coasting intervals; all three interval-ICV upper bounds fall below the pre-registered $0.15$ line. Robustness is in the protocol across seed and shard, not in a single lucky checkpoint.}
\label{tab:s1b}
\begin{tabular}{llcccc}
\toprule
confirmation & ckpt & cal\,$\to$\,test shard & $\epscert$ & $\Teq$ & held-out $U_{95}$ \\
\midrule
fresh-blind   & r2 & 000\,$\to$\,001 & 1.15 & 3 & \textbf{0.092} \\
seed confirm  & r1 & 000\,$\to$\,001 & 1.10 & 2 & \textbf{0.139} \\
shard confirm & r2 & 000\,$\to$\,002 & 0.95 & 2 & \textbf{0.095} \\
\bottomrule
\end{tabular}
\end{table}

\noindent Table~\ref{tab:s1b} reports the confirmation grid. All three held-out interval-ICV upper bounds are below the $0.15$ line; each $\epscert$ is independently calibrated on its calibration split. The protocol --- not a single lucky checkpoint --- is robust across seed and shard. \emph{This establishes that the sensing-clock primitive instantiates on a real equivariant world model, not only a synthetic bench.} An exploratory structure ablation (Appendix~\ref{app:l2}) further indicates that this certifiability is what the equivariant structure buys: matched non-equivariant models mostly admit no certified deadline at all, in 2D by one-step error saturation and in 3D by a frozen (motion-insensitive) representation.

\subsection{Stage 2A --- reactive contrast (matched budget $\approx0.068$)}
\label{sec:r-2a}

% Table 3 -- Stage 2A reactive contrast (matched budget ~0.068)
\begin{table}[t]
\centering
\small
\caption{Stage 2A --- reactive contrast at matched budget $\approx0.068$. The certified clock (Eq-spec) is valid on the deployment distribution ($U_{95}=0.042$) and reduces eventful-tail violations relative to exact-mixture expected-belief scheduling ($0.163$ vs $0.364$); MB-EIG needs $\approx3\times$ budget to recover the tail protection. Risk-sensitive (MB-CVaR) and oracle-robust (MB-WorstCase) exact-model schedulers track Eq-spec closely. The false-exclusion decomposition of the residual tail is reported inline in \S5.2.}
\label{tab:s2a}
\begin{tabular}{lccp{0.30\linewidth}}
\toprule
Policy & overall ICV $U_{95}$ (validity) & eventful (tail) ICV $U_{95}$ & role \\
\midrule
\textbf{Eq-spec} & \textbf{0.042} & 0.163 & certified cue-conditioned clock \\
\textbf{MB-EIG}  & 0.062 & \textbf{0.364} & expected-belief reactive \\
Uniform          & 0.131 & 0.715 & periodic baseline \\
MB-CVaR          & 0.045 & 0.165 & risk-sensitive ($\approx$ Eq-spec) \\
MB-WorstCase     & 0.040 & 0.169 & oracle robust upper bound ($\approx$ Eq-spec) \\
\bottomrule
\end{tabular}
\end{table}

\noindent Table~\ref{tab:s2a} reports the matched-budget contrast; Figure~\ref{fig:tail} visualizes the eventful-tail separation. A \textbf{false-exclusion decomposition} of the Eq-spec eventful intervals explains its residual tail. Here $p_R(c)$ is the scheduler's posterior that the upcoming interval is in the risky (eventful) mode given its pre-interval cue $c$, and $p_{\mathrm{adm}}$ is the admission threshold, calibrated as the $\alpha_{\mathrm{mode}}=0.05$ quantile of $p_R$ over truly risky calibration intervals; by construction at most an $\alpha_{\mathrm{mode}}$-mass of truly risky intervals falls below $p_{\mathrm{adm}}$ --- the designed false-exclusion budget.
% TODO: promote to a numbered table if desired (kept inline to preserve Table 1--4 numbering).
\begin{center}
\small
\begin{tabular}{p{0.30\linewidth}ccp{0.38\linewidth}}
\toprule
subset & $n$ & eventful ICV & interpretation \\
\midrule
$p_R<p_{\mathrm{adm}}$ (calibrated false-exclusion mass) & 92 & 0.75 & pre-registered $\alpha_{\mathrm{mode}}=0.05$ coverage cost \\
$p_R\ge p_{\mathrm{adm}}$ (admitted risky set) & 509 & \textbf{0.031} & deadline controls the tail \\
all intervals (deployment) & --- & \textbf{0.042} & primary validity \\
\bottomrule
\end{tabular}
\end{center}

\noindent $\Rightarrow$ Eq-spec's residual tail leakage is the \emph{designed} false-exclusion budget, not a deadline failure. \textbf{Honest caveat (do not over-read):} the separation is against \emph{expected-belief} MB-EIG specifically; risk-sensitive (MB-CVaR, $0.165$) and robust (MB-WorstCase, $0.169$) exact-model schedulers track Eq-spec closely --- expected in the exact-model limit, and the reason the replaceability question is deferred to the learned regime (\S5.4). Appendix~\ref{app:mbeig} records the exact switch-time mixture form of the MB-EIG baseline (the strongest expected-belief scheduler in this bed).

\begin{figure}[t]
\centering
\includegraphics[width=0.9\linewidth]{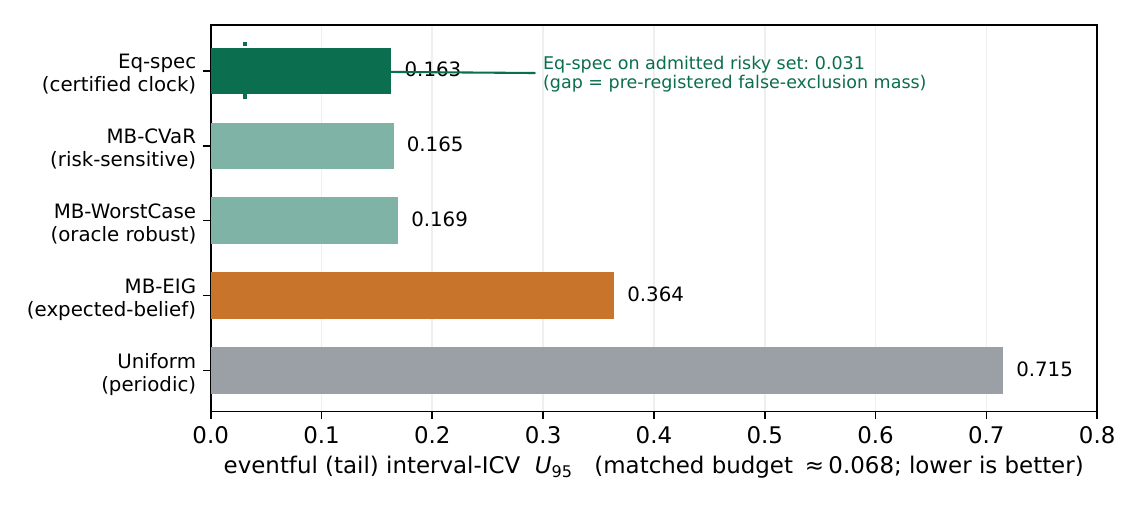}
\caption{\textbf{Stage 2A eventful-tail separation at matched budget ($\approx0.068$).} Eventful (tail) interval-ICV $U_{95}$ per policy; lower is better. The certified clock (Eq-spec) reduces tail violations relative to exact-mixture expected-belief scheduling (MB-EIG) and the periodic baseline, while risk-sensitive (MB-CVaR) and oracle-robust (MB-WorstCase) schedulers track it closely, as expected in the exact-model limit. The dashed marker shows Eq-spec's eventful ICV restricted to its admitted risky set ($0.031$): the residual gap to the full bar is the pre-registered false-exclusion mass, not deadline failure.}
\label{fig:tail}
\end{figure}

\subsection{Stage 2B --- lead-time contrast (NULL)}
\label{sec:r-2b}

On a state-dependent late-sensing bench (matched budget $\approx0.50$, last-safe window $\approx2$--$4$ steps):
\begin{center}
\small
\begin{tabular}{lc}
\toprule
Policy & $\Pr[\Lsafe>0\mid\text{hazard}]$ \\
\midrule
Eq-spec          & 0.955 \\
Uniform          & 0.957 \\
CUSUM (budgeted) & 0.947 \\
\bottomrule
\end{tabular}
\end{center}

\noindent No lead-time separation: at matched budget, periodic and residual-reactive schedulers are as timely as the certified clock --- indeed the uniform baseline is marginally higher. The pre-checks pass (the bench is non-degenerate), but the geometry that satisfies them forces a high sensing rate ($B\approx0.5$, every ${\sim}2$ steps) over a short horizon, where periodic sensing already covers the hazard. \textbf{Stage 2B returns a null result.} A post-hoc structural analysis (Appendix~\ref{app:bench2b}) sharpens this null: with homogeneous starts, any a-priori deadline coincides step-for-step with budget-matched periodic polling, and in this 1D family the geometry that would create deadline diversity simultaneously destroys the late-onset structure reactive baselines need to fail --- the lead-time question requires richer hazard structure than this bench can pose.

\subsection{Stage 2C --- learned-regime replaceability (validity-only warning)}
\label{sec:r-2c}

The feasibility audit (2C-0, read-only) finds the VN-JEPA coasting corpus has \textbf{no native hazard / last-safe / actionable-intervention semantics}, so the full three-axis replaceability test (validity + lead-time + budget) is \textbf{blocked}; only validity/budget replaceability is testable. We also confirm by code inspection that the deployed deadline is exactly $\Tdrift$ (drift-only); $\lambda$ never enters the deployed horizon. The validity/budget test (2C-V) compares against non-equivalent, non-spectral comparators at matched budget on three confirmed splits\footnote{The Eq-spec-drift $U_{95}$ values re-evaluate the same frozen protocol in an independent run; the $0.140$ here vs $0.139$ in Table~\ref{tab:s1b} on r1/sh1 is a rounding-level difference.} (comparator C, a drift-only horizon, is an equivalent implementation of the deployed clock and is therefore excluded as non-distinct):
\begin{center}
\small
\setlength{\tabcolsep}{5pt}
\begin{tabular}{lcccc}
\toprule
split & Eq-spec-drift $U_{95}$ & A: budgeted-CUSUM & B: empirical-conformal & D: robust-CVaR \\
\midrule
r2/sh1 ($\Teq=3$) & 0.092 & over-budget & \textbf{matched (0.093)} & over-budget \\
r1/sh1 ($\Teq=2$) & 0.140 & matched     & matched                 & matched \\
r2/sh2 ($\Teq=2$) & 0.095 & matched     & matched                 & matched \\
\bottomrule
\end{tabular}
\end{center}

\noindent A purely empirical conformal horizon (using neither $\lambda$ nor the drift envelope) \textbf{matches the deployed clock on validity and budget on all three splits}; CUSUM and robust-CVaR match on the short-horizon splits. \textbf{Verdict: validity-only replaceability warning.} The deployed drift-envelope clock is empirically replaceable on validity/budget by a non-spectral conformal horizon in this short-horizon regime; full non-replaceability is unresolved (lead-time semantics absent). This does not invalidate the sensing-clock primitive; it shows that, in the short-horizon full-resensing VN-JEPA regime, the deployed clock's validity guarantee can be reproduced by a simpler conformal horizon. The primitive's unresolved question is non-replaceability in longer-horizon or partial-sensing regimes.

\subsection{E0 --- partial-reset spectral-edge exploration (NULL)}
\label{sec:r-e0}

We additionally explored a partial-reset regime designed to activate the $C e^{\lambda h}e_0$ term. The spectral term became nonzero, but after correcting two false-positive traps --- linearized propagation artifacts and unmatched sensing budgets --- it did not yield a clean budget-matched advantage over empirical conformal horizons. E0 therefore narrows the role of $\lambda$: in the current VN-JEPA regime it is an audit interface rather than an empirically necessary scheduling term. Table~\ref{tab:boundary} summarizes the claim boundary.

% Table 4 -- boundary / negative-results summary (KEEP IN MAIN TEXT, not appendix)
\begin{table}[t]
\centering
\small
\caption{Boundary / negative-results summary. These negative results narrow the claim rather than invalidate the primitive: the certified sensing-clock primitive and the drift-aware deployment method are established; empirical dominance of the spectral term over non-spectral schedulers is deferred to longer-horizon / partial-sensing / lead-time regimes.}
\label{tab:boundary}
\begin{tabular}{p{0.30\linewidth}p{0.27\linewidth}p{0.34\linewidth}}
\toprule
Question & Result & Consequence \\
\midrule
Lead-time over residual-reactive?           & Stage 2B NULL (\S5.3)      & not established \\
Learned validity/budget non-replaceability? & Stage 2C-V warning (\S5.4) & empirical conformal matches deployed clock \\
Spectral-term deployment edge?              & E0 NULL (\S5.5)           & $\lambda$ remains an audit interface in this regime \\
Full C2 (empirical dominance)?              & pending / not established  & no dominance claim over reactive / conformal / robust schedulers \\
\bottomrule
\end{tabular}
\end{table}

\noindent These negatives jointly define the claim boundary: the primitive and drift-aware method are established; \emph{empirical dominance of the spectral term over non-spectral schedulers is not}, and is deferred to longer-horizon / partial-sensing / lead-time regimes.

\section{Discussion}
\label{sec:disc}

\paragraph{What is established.}
(i) A certified sensing-clock primitive that controls interval-simultaneous certificate violation; (ii) a drift-aware deployment method, with the empirical finding that on-manifold $\lambda$ alone over-states the horizon and the native rollout-drift envelope carries the deployed guarantee; (iii) instantiation on a real frozen 3D VN-JEPA equivariant world model, confirmed across seed and shard; (iv) a reactive-contrast theorem-bed showing expected-belief scheduling under-prices the eventful tail at matched budget.

\paragraph{What is not established.}
(i) Actionable lead-time over residual-reactive schedulers (Stage 2B null on exact synthetic; not testable on native VN-JEPA traces). (ii) Non-replaceability: in the learned short-horizon regime a non-spectral empirical conformal horizon matches the deployed clock on validity and budget (Stage 2C-V warning). (iii) An empirical edge for the spectral term itself (E0 null). We explicitly do \textbf{not} claim the spectral clock empirically dominates reactive, conformal, or robust schedulers.

\paragraph{When might $\lambda$ matter?}
The spectral term is dormant precisely because deployment fully re-senses ($e_0\approx0$) over short horizons. Regimes that could re-activate it --- non-trivial / state-dependent post-sense error (partial sensing), longer coasting horizons, or imperfect-model settings where audited spectra and drift envelopes diverge from empirical conformal calibration --- are natural future work. Our partial-reset probe (\S5.5) activated the term but found no budget-matched edge, so we present this cautiously.

\paragraph{Why this is still a new primitive.}
Even where simpler conformal horizons match the deployed clock numerically, the contribution is the \emph{framing and method}: turning a certified validity horizon into an operational re-sensing deadline, the drift-aware correction that makes it deployable on a real equivariant world model, and the pre-registered reactive-contrast methodology (including the eventful-tail / false-exclusion analysis) under which expected-belief scheduling is shown to under-price tail risk. The certificate-as-clock interface --- connecting certified world models to action --- is the durable contribution.

\paragraph{Anticipated reviewer questions.}

\emph{Q1 --- If a non-spectral conformal horizon reproduces the deployed clock (\S5.4), why is this valuable?}
The deliverable is the primitive, the method, and the methodology --- not a number. The certificate-as-clock interface, the drift-aware correction that makes the horizon deployable on a real equivariant world model, and the pre-registered reactive-contrast protocol (eventful-tail / false-exclusion decomposition) stand independently of whether a conformal horizon happens to match the deadline in one regime. Replaceability is demonstrated only on validity and budget, only in the short-horizon full-resense regime, and only against schedulers that themselves lack a predictive-validity certificate; we report it as a warning, not a refutation, and name the regimes where the question reopens. The primitive also has a theoretical identity: in a Gaussian toy the certified clock and the epistemic-threshold scheduler of active inference parameterize the same schedule family, and the deployed drift-aware clock is that family's \emph{calibrated} member (Proposition~\ref{prop:efe}, Appendix~\ref{app:efe}).

\emph{Q2 --- If the deployed clock is drift-only, why keep the spectral / Lyapunov framing?}
The title foregrounds \emph{drift-aware deadlines} precisely because the deployed term is the drift envelope. The spectral rate is retained as the \textbf{derivation and audit interface}: it is how a certificate horizon is read off an equivariant world model, it supplies the on- vs off-manifold expansion diagnostic ($\lambda\approx0.07$ on-manifold vs $\approx3.5$ off-manifold), and it is the term that partial-sensing or longer-horizon regimes would reactivate ($e_0\not\approx0$). We are explicit (\S3, \S5.5) that $\lambda$ is \emph{not} the deployed active term in the full-resense regime; consistently, the shared certified/epistemic countdown of the Gaussian toy is exactly the spectral horizon $T_\lambda$ (Corollary~\ref{cor:countdown}), while the exploratory ablation finds the audited rate mis-states the deployable horizon in \emph{opposite} directions across regimes (Appendix~\ref{app:l2}) --- both facts point to the calibrated envelope as the deployment term and to $\lambda$ as the derivation interface.

\emph{Q3 --- Stage 2A only separates from expected-belief MB-EIG; risk-sensitive (MB-CVaR) and robust (MB-WorstCase) schedulers tie. Is the win real?}
Yes, with a stated scope. In the exact-model limit, risk-sensitive and robust schedulers provably approach the certified clock --- MB-WorstCase is essentially the oracle robust reference, so a tie with it is the \emph{expected} and desirable outcome, not a weakness. The honest claim is narrow: \emph{expected-belief} (mean-information-gain) scheduling under-prices the eventful tail ($0.36$ vs $0.16$ at matched budget), and the certified clock reaches the robust/risk-sensitive frontier without oracle access to the worst case. Whether the certified clock can \emph{separate} from CVaR / worst-case once exact-model equivalence breaks is exactly the learned-regime question Stage 2C defers.

\section{Related Work}
\label{sec:related}

We position the primitive against five neighboring families; the carve is narrow --- we make the validity horizon \emph{itself} the trigger --- and we do not claim to outperform these families in general.

\paragraph{Active sensing / value-of-information / POMDP with guarantees.}
Measurement simplification in $\rho$-POMDPs with performance guarantees~\citep{yotam2023rhopomdp} schedules and simplifies observations by belief-space value while bounding the loss. We differ by deriving a proactive deadline from an audited world-model validity horizon rather than from online belief-space planning: the trigger is certified coasting validity, not task value or belief utility.

\paragraph{Reactive information-gain / entropy scheduling.}
Information-driven active perception selects sensor queries by the $k$-step conditional entropy of future-state safety under a query budget~\citep{udupa2026kstep}. These schedulers act on predicted uncertainty as it accrues. Our Stage~2A shows a separation from \emph{expected-belief} (mean-information-gain) scheduling (MB-EIG) specifically, on eventful-tail validity at matched budget; we do \emph{not} claim a separation from all reactive schedulers --- risk-sensitive and robust variants track our clock in the exact-model limit (\S5.2).

\paragraph{Conformal / calibrated uncertainty.}
Conformal prediction gives distribution-free, finite-sample coverage~\citep{angelopoulos2021conformal} and underlies a strong non-spectral baseline. We are explicit that conformal horizons are strong here: in the frozen VN-JEPA short-horizon regime, an empirical conformal horizon matches our deployed clock on validity and budget (Stage~2C-V, \S5.4). Our contribution is the certificate-as-clock \emph{framing and protocol}, not numerical superiority over a conformal horizon in this regime; the conformal comparison is reported as a warning, not a defeat.

\paragraph{Active inference / expected free energy.}
Belief-space control via active inference schedules measurements by an epistemic (expected-free-energy) objective~\citep{sargun2026cancer}. Such schedulers are objective-driven and, in this line, domain-specific and certificate-free; ours is a certificate-grounded interval guarantee rather than a free-energy heuristic. The bridge is in fact formal: in a Gaussian toy the two trigger rules parameterize the same schedule family, with the deployed drift-aware clock as its calibrated member (Appendix~\ref{app:efe}).

\paragraph{World-model verification and certified world models.}
A growing line verifies imagined rollouts or controls robustly: closed-loop world-model verification~\citep{geng2026dwm}, event-verified world models~\citep{wang2026evwm}, and robust MPC for certified manipulation~\citep{li2026robustmpc}. This work builds on a certified-horizon line that supplies the faithful-certificate and prediction-horizon machinery~\citep{wang2026certifiedwm}; we turn that horizon into an operational re-sensing deadline. Versus detector-style residual monitors: a detector alarms \emph{after} seeing a residual, our clock senses \emph{before} one arrives. Table~\ref{tab:related} (Appendix~\ref{app:relmap}) maps the families.

\section*{Reproducibility Statement}
All confirmatory results (Stages 0A--2C) come from a pre-registered, frozen experiment line: certificate scales, tolerances, drift envelopes, and admission thresholds are fixed on calibration splits, and each test split is evaluated once with no retuning. Exploratory results are labeled as such and were replicated on fresh splits before being reported: the partial-reset probe (Appendix~\ref{app:e0}), the structure ablation (Appendix~\ref{app:l2}), and the lead-time geometry analysis (Appendix~\ref{app:bench2b}); negative results are reported rather than discarded. Remaining future work: longer-horizon / partial-sensing $e_0$ regimes on the real model, and richer hazard structure for the lead-time question (Appendix~\ref{app:bench2b}).

\bibliographystyle{plainnat}
\bibliography{refs}

\begin{thebibliography}{8}
\providecommand{\natexlab}[1]{#1}
\providecommand{\url}[1]{\texttt{#1}}
\expandafter\ifx\csname urlstyle\endcsname\relax
  \providecommand{\doi}[1]{doi: #1}\else
  \providecommand{\doi}{doi: \begingroup \urlstyle{rm}\Url}\fi

\bibitem[Angelopoulos and Bates(2021)]{angelopoulos2021conformal}
Anastasios~N. Angelopoulos and Stephen Bates.
\newblock A gentle introduction to conformal prediction and distribution-free
  uncertainty quantification.
\newblock arXiv:2107.07511, 2021.

\bibitem[Geng et~al.(2026)Geng, Zhou, Zhang, Pan, Tran, and
  Ruchkin]{geng2026dwm}
Yuang Geng, Zhuoyang Zhou, Zhongzheng Zhang, Siyuan Pan, Hoang-Dung Tran, and
  Ivan Ruchkin.
\newblock Deterministic world model for closed-loop verification of end-to-end
  vision-based controller.
\newblock arXiv:2512.08991, 2026.

\bibitem[Li et~al.(2026)Li, Fang, Polisetti, Song, and Chou]{li2026robustmpc}
Wei-Chen Li, Jeffrey Fang, Sasanka Polisetti, Yuexi Song, and Glen Chou.
\newblock Robustness without wrinkles: Parallel simulation and robust {MPC} for
  certified deformable manipulation.
\newblock arXiv:2606.14188, 2026.

\bibitem[Sargun et~al.(2026)Sargun, Tulay, and Koksal]{sargun2026cancer}
Deniz Sargun, H.~Bugra Tulay, and C.~Emre Koksal.
\newblock Belief-space control for personalized cancer treatment via active
  inference.
\newblock arXiv:2606.10376, 2026.

\bibitem[Udupa and Fu(2026)]{udupa2026kstep}
Sumukha Udupa and Jie Fu.
\newblock Information-driven active perception for k-step predictive safety
  monitoring.
\newblock arXiv:2603.23450, 2026.

\bibitem[Wang(2026)]{wang2026certifiedwm}
Hongbo Wang.
\newblock Certified world models: Predictability across configuration, horizon,
  and resolution.
\newblock arXiv:2606.13092, 2026.

\bibitem[Wang et~al.(2026)Wang, Jie, Yan, Zhou, and Heng]{wang2026evwm}
Kailin Wang, Haoxiang Jie, Yaoyuan Yan, Jiacheng Zhou, and Zhiyou Heng.
\newblock {EV-WM}: Event-verified world models for long-horizon robotic
  manipulation.
\newblock arXiv:2606.13053, 2026.

\bibitem[Yotam and Indelman(2023)]{yotam2023rhopomdp}
Tom Yotam and Vadim Indelman.
\newblock Measurement simplification in {$\rho$-POMDP} with performance
  guarantees.
\newblock arXiv:2309.10701, 2023.

\end{thebibliography}

\appendix
\section{E0: Partial-Reset Exploration --- Protocol, Traps, and Budget-Matched Results}
\label{app:e0}

\textbf{Status.} Exploratory; \emph{not} pre-registered; not claim-bearing. We report it in full for transparency, because the two failure modes below are easy to fall into when evaluating spectral scheduling terms, and because a transparent negative is more informative than an asserted one.

\paragraph{Protocol.}
To activate the spectral term $C e^{\lambda h} e_0$ (dormant under full re-sensing, \S3), each sense applies a \emph{partial reset}: the post-sense latent retains a fraction of the pre-sense coasting error, $e_0^{(i)} = \rho\, e_{\mathrm{pre}}^{(i)}$ with $\rho \in \{0.5, 0.8\}$, so $e_0$ varies per coasting interval (state-dependent, not a constant offset). Policies compared on the frozen VN-JEPA corpus: the drift-aware clock with the spectral term active ($\lambda$+drift), the drift-only clock, and the empirical conformal horizon. Coasting predictions are re-rolled from the perturbed post-sense latent, and violations are always measured against the true $\epscert$.

\paragraph{Trap 1: linearized propagation (circular).}
A first implementation propagated the perturbed error analytically, $r_h \approx r_h^{\mathrm{full\text{-}reset}} + e_0 e^{\lambda h}$. This bakes the spectral model into the data-generating process, so the $\lambda$-aware policy wins by construction. Fix: faithfully re-roll the frozen predictor from the perturbed latent; the apparent advantage disappears.

\paragraph{Trap 2: unmatched budget.}
With certified triggers left free, the $\lambda$+drift policy senses almost every step under partial reset (budget $0.87$--$0.99$ vs $0.34$ for drift-only), producing an apparent tail win that is purely purchased budget. Fix: throttle every policy to a matched target budget $B^*$ by relaxing its trigger with a factor $\gamma \ge 1$. Throttled variants are \emph{allocation diagnostics}, not certified deployments (relaxing the trigger voids the validity guarantee); the comparison is interval-ICV at matched budget against the true $\epscert$.

\paragraph{Budget-matched results.}
Table~\ref{tab:e0} reports held-out interval-ICV upper bounds ($U_{95}$) at matched budget.

\begin{table}[h]
\centering
\small
\caption{E0 budget-matched partial-reset results (exploratory; not pre-registered). The spectral term is active ($e_0 \neq 0$), yet $\lambda$+drift shows no clean budget-matched edge over the empirical conformal horizon: at $B^*{=}0.333$ conformal is lower; at $B^*{=}0.5$ the nominal difference ($0.020$) is below the pre-declared $0.05$ promotion margin. At $\rho{=}0.8$ the post-sense error is too large for any throttle $\gamma \le 2.76$ to reach the target budget, so the cell is infeasible rather than evidence.}
\label{tab:e0}
\begin{tabular}{cccccl}
\toprule
$\rho$ & $B^*$ & $\lambda$+drift $U_{95}$ & drift-only $U_{95}$ & conformal $U_{95}$ & verdict \\
\midrule
0.5 & 0.333 & 0.114 & 0.218 & \textbf{0.102} & NULL (conformal $\le$ $\lambda$+drift) \\
0.5 & 0.500 & \textbf{0.057} & --- (budget mismatch) & 0.077 & NULL (margin $0.020 < 0.05$) \\
0.8 & all   & --- & --- & --- & infeasible ($\gamma \le 2.76$ cannot reach $B^*$) \\
\bottomrule
\end{tabular}
\end{table}

\paragraph{Verdict.}
The partial-reset regime does activate the spectral term, but after removing both traps there is no budget-matched advantage over a non-spectral empirical conformal horizon. This is the basis for the E0 row of Table~\ref{tab:boundary} and for the narrowed reading of $\lambda$ as an audit interface in this regime (\S5.5). Promotion of any future positive in this direction requires fresh-split / fresh-seed replication under a pre-declared margin.

\section{The Sensing Countdown as a Calibrated Epistemic-Threshold Scheduler}
\label{app:efe}

This appendix gives the sensing-clock primitive a theoretical identity: in a Gaussian toy, the certified clock and the epistemic (expected-information-gain) scheduler of active inference are the \emph{same} family of sensing schedules, and the deployed drift-aware clock is that family's \emph{calibrated} member. The result is deliberately narrow (Remarks~\ref{rem:vector}--\ref{rem:scope}); it does not rescue the empirical replaceability boundary of \S5.4--5.5.

\paragraph{Setup.}
A scalar latent coasts as $z_{h+1} = a\,z_h + w_h$, $w_h \sim \mathcal N(0,\sigma_w^2)$. After a sense the agent holds $b_0=\mathcal N(\mu_0,P_0)$ and propagates open-loop, $b_h=\mathcal N(\mu_h,P_h)$ with $P_{h+1}=a^2P_h+\sigma_w^2$. A sense at $h$ returns $o_h=z_h+v_h$, $v_h\sim\mathcal N(0,\sigma_o^2)$. Under model correctness the coasting error satisfies $e_h\sim\mathcal N(0,P_h)$, so with $q_\alpha$ the standard-normal $(1-\alpha/2)$-quantile the certified clock triggers at
\begin{equation}
\tau_{\rm cert}(\epsilon)=\inf\{h\ge1:\ q_\alpha^2 P_h>\epsilon^2\},
\end{equation}
while the epistemic-threshold scheduler triggers when the expected information gain of sensing,
$\mathrm{IG}(h)=I(z_h;o_h)=\tfrac12\ln\!\big(1+P_h/\sigma_o^2\big)$,
first exceeds a threshold $c$:
\begin{equation}
\tau_{\rm EFE}(c)=\inf\{h\ge1:\ \mathrm{IG}(h)>c\}.
\end{equation}

\begin{proposition}[threshold equivalence]
\label{prop:efe}
Fix $\alpha\in(0,1)$ and $\sigma_o^2>0$, and let
$c(\epsilon)=\tfrac12\ln\!\big(1+\epsilon^2/(q_\alpha^2\sigma_o^2)\big)$.
Then for \emph{every} variance path $(P_h)_{h\ge1}$ --- regardless of the dynamics generating it ---
$\tau_{\rm EFE}(c(\epsilon))=\tau_{\rm cert}(\epsilon)$,
and $\epsilon\mapsto c(\epsilon)$ is a strictly increasing bijection of $(0,\infty)$ onto itself. Certified tolerances and epistemic thresholds therefore parameterize the same family of sensing schedules.
\end{proposition}

\begin{proof}
For each $h$: $\mathrm{IG}(h)>c(\epsilon)\iff\ln(1+P_h/\sigma_o^2)>\ln(1+\epsilon^2/(q_\alpha^2\sigma_o^2))\iff P_h>\epsilon^2/q_\alpha^2\iff q_\alpha^2P_h>\epsilon^2$, by strict monotonicity of $x\mapsto\ln(1+x/\sigma_o^2)$. The trigger events coincide at every $h$, hence so do the first-crossing times (including $\tau=\infty$ on both sides). Bijectivity is immediate.
\end{proof}

\begin{corollary}[shared countdown = the spectral horizon]
\label{cor:countdown}
Under the expanding recursion ($a^2>1$), $P_h=a^{2h}\big(P_0+\kappa\big)-\kappa$ with $\kappa=\sigma_w^2/(a^2-1)$ is strictly increasing, so both schedulers reduce to one deterministic countdown whose deadline is, up to the additive offset $\kappa$,
\begin{equation}
T(\epsilon)\;\approx\;\frac{1}{\lambda}\,\ln\frac{\epsilon}{q_\alpha\,\tilde e_0},
\qquad \lambda=\ln a,\quad \tilde e_0=\sqrt{P_0+\kappa},
\end{equation}
i.e.\ exactly the spectral horizon $T_\lambda(\epsilon)\sim\ln(\epsilon/e_0)/\lambda$ of \S3.
\end{corollary}

\begin{remark}[the deployed clock is the calibrated member]
\label{rem:calibrated}
Both sides of Proposition~\ref{prop:efe} take the \emph{model-trusting} variance $P_h$ as input. Under model error the realized coasting error is no longer $\mathcal N(0,P_h)$: the epistemic scheduler keeps triggering on the model's $P_h$ (mis-timed), whereas the deployed clock replaces $q_\alpha\sqrt{P_h}$ by the held-out calibrated drift envelope $b_h^{\mathrm{UCB}}$ (\S3). The two schedules then differ by exactly that calibration correction:
\emph{deployed drift-aware clock $=$ epistemic trigger $+$ distribution-free calibration.}
\end{remark}

\begin{remark}[vector / block structure]
\label{rem:vector}
In higher dimension the epistemic term is $\tfrac12\ln\det(I+P_h\Sigma_o^{-1})$ --- a \emph{sum} over spectral directions --- while the block certificate norm $\max_b\norm{e_b}/s_b$ is a \emph{max} functional. The two threshold families coincide exactly only in the scalar/isotropic case; in general they are directionally aligned (each bounds the other up to block-count factors) but not identical.
\end{remark}

\begin{remark}[scope]
\label{rem:scope}
The equivalence concerns the epistemic-only, greedy-threshold form of expected-free-energy scheduling on a fixed observation channel --- not full POMDP planning, and not the pragmatic/risk terms. It confers a theoretical identity on the primitive; it does \emph{not} contradict the empirical findings that a non-spectral conformal horizon matches the deployed clock on validity and budget in the short-horizon regime (\S5.4--5.5).
\end{remark}

\section{Exploratory Structure Ablation: Equivariant vs Plain World Models}
\label{app:l2}

\textbf{Status.} Exploratory; \emph{not} pre-registered; each finding replicated on a fresh split before being reported here. This appendix does not alter the contributions or the boundary of the main text; it addresses a natural reviewer question --- \emph{what does structure buy for the sensing clock?} --- with the same certified-clock protocol run on matched pairs of equivariant and non-equivariant world models.

\paragraph{Protocol.}
For each frozen checkpoint: audit the on-manifold spectrum, fit the native drift envelope $b_h^{\rm UCB}$ (P95) on a calibration split, calibrate $\epsilon_{\rm cert}$ per model (smallest tolerance with calibration interval-ICV $U_{95}\le0.10$ at the induced deadline, $1<T<H$), under a per-model scale-free certificate norm $\norm{e}_2/s$ with $s$ the calibration median one-step target displacement. A model--split pair is \emph{certifiable} if such an $(\epsilon,T)$ exists. In 3D, the tolerance grid is additionally capped by a latent-diameter anchor (a certification tolerating errors beyond the representation cloud's own diameter is meaningless); this cap also removes a grid-resolution pathology caused by exploding-rollout seeds.

\paragraph{2D (capacity-matched pair).}
Paired training wave, $167{,}360$ vs $167{,}424$ parameters; certifiability assessed on two disjoint splits.

\begin{center}
\small
\begin{tabular}{lccc}
\toprule
arm & certifiable (split 1 / split 2) & sat.\ ratio $e_{12}/e_1$ & mechanism \\
\midrule
equivariant (7 seeds) & \textbf{7/7 \ / \ 7/7} & 5.5--11.7 & monotone error hierarchy \\
plain (6 seeds)       & 2/6 \ / \ 2/6           & 1.3--3.0  & one-step error saturation \\
\bottomrule
\end{tabular}
\end{center}

The certifiable equivariant models deploy at $T=2$ with held-out interval-ICV $U_{95}\in[0.047,0.088]$ (all valid). The plain failures are not latent collapse (latent variance is comparable across arms); the plain predictor's open-loop error saturates within one step ($b_1\!\approx\!2.3\to b_{12}\!\approx\!2.6$), leaving no monotone hierarchy for the clock to calibrate a deadline on.

\paragraph{3D (same corpus and protocol as \S5.1; plain is $7.9\times$ larger).}
VN ($52$K params) vs plain PointNet-style ($414$K), three seeds each, two fresh shards never used elsewhere.

\begin{center}
\small
\begin{tabular}{lccc}
\toprule
arm & certifiable (shard A / B) & deployed form & mechanism \\
\midrule
VN (equivariant) & \textbf{3/3 \ / \ 3/3} & $T=2$--$3$, $\epsilon\approx2.5$--$2.7\ \ll$ diam $10$--$15$ & clean hierarchy \\
plain            & 0/3 \ / \ 1/3$^{\dagger}$ & --- & \emph{frozen representation} \\
\bottomrule
\end{tabular}
\end{center}

Five of six plain runs have median one-step target-latent displacement $s\le2\times10^{-13}$: the encoder's latent does not respond to motion at all --- the extreme form of hierarchy collapse, and a retroactive explanation for the near-zero spectral audits this model family produced. ($^{\dagger}$The remaining seat is a numerical ghost: $s$ in the $(10^{-12},10^{-5})$ gray zone inflates the normalized domain by $\sim\!10^{11}$; under a physical degeneracy threshold it is also frozen. Both readings leave the contrast intact.)

\paragraph{Reading, and a bonus about $\lambda$.}
In both regimes the contrast is not about horizon \emph{length}: it is about \emph{certifiability} --- \textbf{structure buys the monotone error hierarchy that a certified deadline calibrates on}. Separately, the audited top rate mis-states the deployable horizon in \emph{opposite directions} across the two regimes (2D: $\lambda\approx5.5$--$7.7$ makes the naive $T_\lambda=1$ \emph{under}-state; 3D: $\lambda\approx0.07$ \emph{over}-states by $\sim\!10\times$, \S3) --- reinforcing that the calibrated drift envelope, not the spectral rate, is the load-bearing deployment term in both.

\section{Stage 2A: The Exact Switch-Time Mixture MB-EIG Baseline}
\label{app:mbeig}

Stage 2A's honesty hinges on the reactive baseline being the \emph{strongest} expected-belief scheduler available in the theorem-bed, so we record its exact form.

\paragraph{Hazard model.}
Each coasting interval carries a latent risk type $u\in\{S,R\}$ emitting a noisy pre-interval cue $c$ with posterior $p_R(c)$. In a risky interval the mode switches from benign to expansive dynamics at a hidden time and the switch is \emph{absorbing}:
\begin{equation}
\Pr(m_{\tau+h}=\text{expansive}\mid u=R)\;=\;1-(1-p_{\rm sw})^h .
\end{equation}

\paragraph{Exact mixture.}
Because the switch is absorbing, the branch space at horizon $h$ is not $2^h$ paths but $h{+}1$: ``switch first occurred at step $s$'' ($s=1,\dots,h$) plus ``no switch by $h$''. MB-EIG therefore evaluates its expected-entropy objective on the \emph{exact} mixture
\begin{equation}
J_{\rm exp}(h,c)\;=\;\sum_{s=1}^{h}\Pr(s\mid c)\,H\!\big(P_h^{(s)}\big)\;+\;\Pr(\text{no switch by }h\mid c)\,H\!\big(P_h^{(0)}\big),
\end{equation}
with absorbing-hazard branch weights
\begin{equation}
\Pr(s\mid c)=p_R(c)\,(1-p_{\rm sw})^{s-1}p_{\rm sw},
\qquad
\Pr(\text{no switch by }h\mid c)=1-p_R(c)\big(1-(1-p_{\rm sw})^h\big),
\end{equation}
where $P_h^{(s)}$ propagates $s{-}1$ benign steps followed by expansive steps and $P_h^{(0)}$ is all-benign. The risk-sensitive baseline (MB-CVaR) computes $\mathrm{CVaR}_\beta$ on the \emph{same} branch distribution $\{\Pr(s\mid c)\}$.

\paragraph{Why moment matching is not allowed as the primary baseline.}
The single-covariance recursion
$P_{h+1}^{\rm mix}=(1-\pi_h)A_bP_h^{\rm mix}A_b^\top+\pi_hA_eP_h^{\rm mix}A_e^\top+Q$, with $\pi_h=p_R(c)(1-(1-p_{\rm sw})^h)$,
collapses the branch distribution to its moments and thereby \emph{weakens} the reactive baseline exactly where the contrast lives (the expansive tail). It is used only as a speed/diagnostic fallback. Consequently the Stage~2A separation (\S5.2) is against the exact-mixture --- i.e.\ the strongest --- expected-belief scheduler, and the residual-tail decomposition attributes Eq-spec's own leakage to its pre-registered false-exclusion mass rather than to deadline failure.

\section{Stage 2B: Bench Geometry, and Why Lead-Time Separation Is Structurally Hard in 1D}
\label{app:bench2b}

\paragraph{Bench.}
A 1D point system $p_{t+1}=p_t+v_t+\mathcal N(0,\sigma_p^2)$, $v_{t+1}=v_t+u_t+\mathcal N(0,\sigma_v^2)$, with an irreversible failure set $\{p\ge p_{\rm fail}\}$ and a hazard zone $[p_{\rm haz},p_{\rm fail})$. A PD controller drives the cart \emph{through} the hazard toward $p_{\rm goal}>p_{\rm fail}$, so every nominal episode is a genuine approach-to-hazard; $t_{\rm last}$ is the last step at which max-braking still avoids failure, and $L_{\rm safe}=t_{\rm last}-\tau_{\rm trigger}$. The Stage~2B configuration (\S5.3) uses matched budget $B\approx0.50$ and $t_{\rm last}\approx2$--$4$.

\paragraph{Post-hoc structural analysis (no additional contrast was run).}
A follow-up geometry search for a low-budget variant ($B\le0.20$, $t_{\rm last}\in[6,12]$, deadline diversity $T_{\rm high}/T_{\rm low}\ge3$, late residual onset) terminated at its pre-declared gate: the gates could not be satisfied simultaneously in this bench family. The diagnosis is structural rather than a tuning failure, and it sharpens the reading of the Stage 2B null:

\begin{enumerate}
\item \emph{Homogeneous starts make the null a theorem.} With a common start state, the certified deadline $T_{\rm spec}(z_0)$ and a uniform poller's period are both episode-constant; at matched budget they coincide step-for-step, so \emph{no} a-priori deadline can separate from constant-period polling. The near-identical lead-time rates of \S5.3 ($0.955$ vs $0.957$) are the expected outcome, not a coincidence.
\item \emph{Heterogeneous starts trade the two required structures against each other.} In this family the certified deadline, the braking deadline, and hazard entry are nearly collinear --- all shifts of arrival time: $T_{\rm spec}\approx t_{\rm entry}+\Delta$, $t_{\rm last}\approx t_{\rm entry}+w$. Creating deadline diversity (spread in arrival times) simultaneously widens the slow tail's braking window $w$, which restores exactly the early-onset structure that lets reactive and periodic baselines succeed. In a six-round calibration sweep the diversity gate never exceeded $2.0\times$ (target $3\times$) while the late-onset gate peaked at $0.69$ (target $0.8$).
\end{enumerate}

\paragraph{Consequence.}
The lead-time question for certified sensing clocks is not settled negatively by \S5.3 alone; rather, \emph{this bench family cannot pose it}. Establishing (or refuting) an actionable lead-time advantage requires hazard structure in which certified deadlines and braking deadlines decouple --- multiple hazard sites, non-monotone approaches, or higher-dimensional dynamics. We record this as an open bench-design problem.

\section{Related-Work Map}
\label{app:relmap}

\begin{table}[h]
\centering
\small
\caption{Related-work map. Each family schedules or verifies under a different trigger and guarantees a different object; the P4-A primitive triggers on a certified, drift-aware validity horizon and guarantees interval-simultaneous certificate validity.}
\label{tab:related}
\setlength{\tabcolsep}{3.5pt}
\begin{tabular}{p{0.205\linewidth}p{0.225\linewidth}p{0.225\linewidth}p{0.205\linewidth}}
\toprule
Family & Trigger & Guarantee object & Our difference \\
\midrule
VoI / POMDP active sensing~\citep{yotam2023rhopomdp} & expected belief value & task reward / information & certified coasting validity \\
Reactive entropy / info-gain~\citep{udupa2026kstep} & predicted uncertainty / entropy & expected belief quality & eventful-tail ICV not controlled by mean-IG \\
Conformal / calibrated UQ~\citep{angelopoulos2021conformal} & calibrated residual & distribution-free coverage & temporal sensing deadline (conformal is a strong replacement here) \\
Active inference / EFE~\citep{sargun2026cancer} & epistemic-value objective & heuristic / objective-driven & certificate-grounded horizon (formal bridge: Appendix~\ref{app:efe}) \\
WM verification~\citep{geng2026dwm,wang2026evwm} & rollout / event check & verified imagined future & validity-horizon-as-clock \\
\textbf{P4-A (ours)} & certified drift-aware deadline & interval-simultaneous ICV & sensing-clock primitive \\
\bottomrule
\end{tabular}
\end{table}

\end{document}